\documentclass{article}
\usepackage[utf8]{inputenc}
\usepackage{style}
\usepackage{physics}
\usepackage{amsthm}
\usepackage{mathtools}
\usepackage{amsfonts}
\usepackage{float}

\title{\vspace{-1cm}More Data Can Hurt for Linear Regression:\\
Sample-wise Double Descent}

\author{Preetum Nakkiran\\Harvard University}
\date{}

\newcommand{\hbeta}{\hat\beta}
\newcommand{\cD}{\mathcal{D}}
\newcommand{\cN}{\mathcal{N}}
\newcommand{\Proj}{\text{Proj}}

\begin{document}

\maketitle

\begin{abstract}
In this expository note we describe a surprising phenomenon
in overparameterized linear regression, where the dimension exceeds the number of samples: there is a regime where the test risk of the
estimator found by gradient descent \emph{increases} with additional samples.
In other words, more data actually hurts the estimator.
This behavior is implicit in a recent line of theoretical works analyzing ``double descent'' phenomena in linear models.
In this note, we isolate and understand this behavior in an extremely simple setting: linear regression with isotropic Gaussian covariates.
In particular, this occurs due to an unconventional type of \emph{bias-variance tradeoff} in the overparameterized regime: the bias decreases with more samples,
but variance \emph{increases}.
\end{abstract}

\section{Introduction}
Common statistical intuition suggests that more data should never harm the performance of an estimator.
It was recently highlighted in \cite{deep} that 
this may not hold for \emph{overparameterized} models:
there are settings in modern deep learning
where training on more data actually hurts. 
In this note, we analyze a simple setting to understand the mechanisms behind this 
behavior.

\begin{figure}[h]
    \centering
    \begin{subfigure}[t]{0.4\textwidth}
        \includegraphics[width=\textwidth]{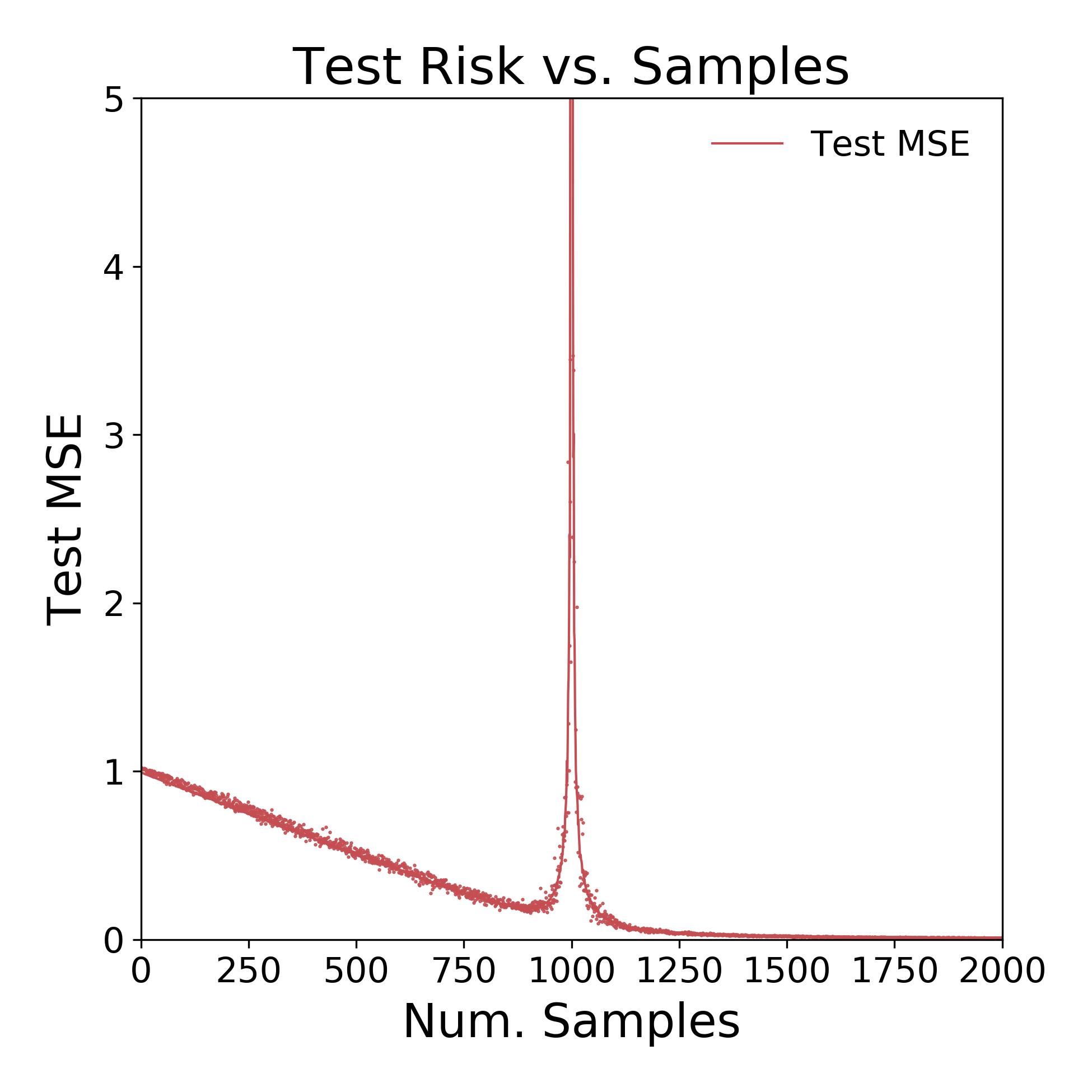}
        \caption{Test MSE for $d =1000, \sigma=0.1$.
    \label{fig:exp}
        }
    \end{subfigure}
    \begin{subfigure}[t]{0.4\textwidth}
        \includegraphics[width=\textwidth]{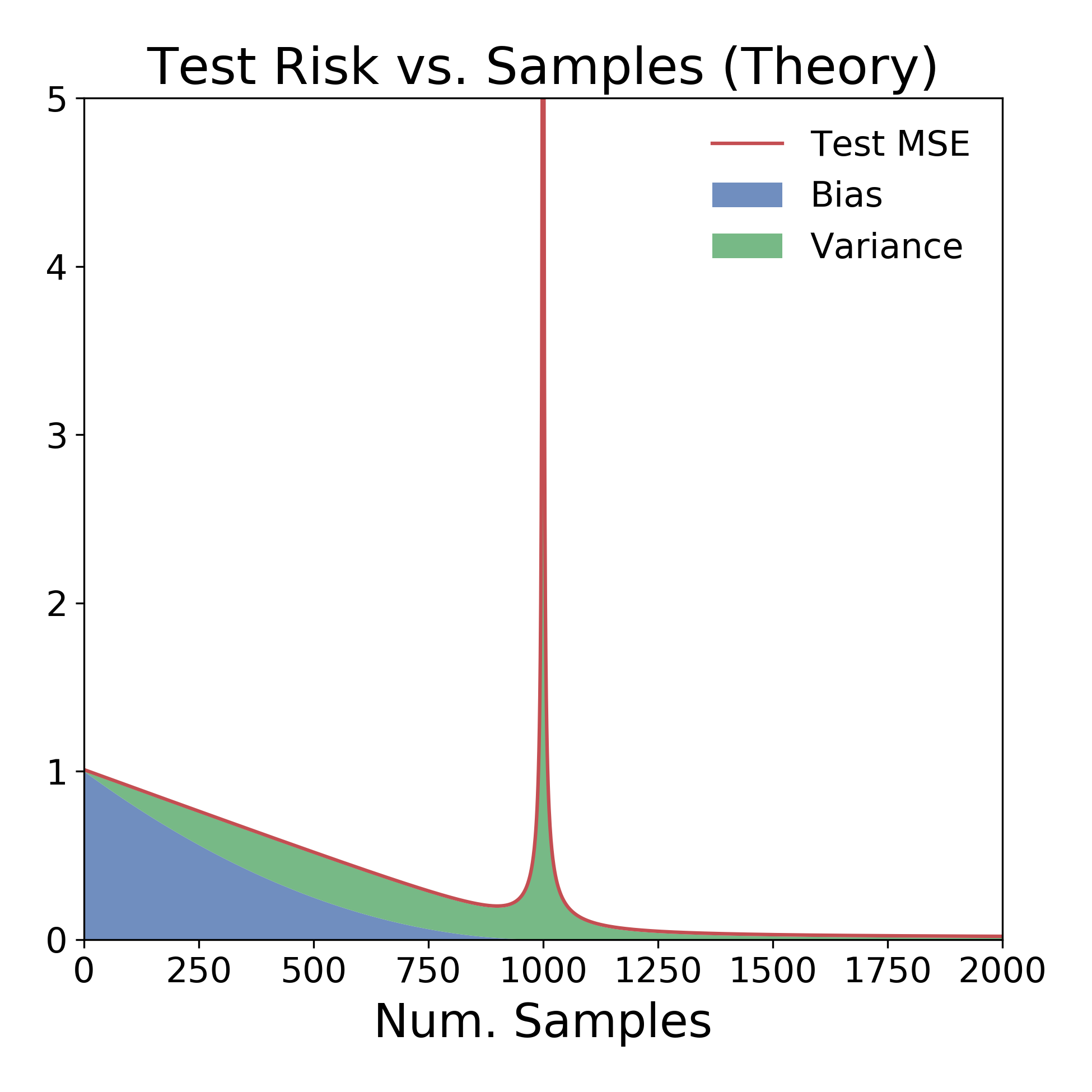}
        \caption{Test MSE in theory for $d =1000, \sigma=0.1$}
        \label{fig:theory}
    \end{subfigure}
    \caption{{\bf Test MSE vs. Num. Train Samples
    for the min-norm ridgeless regression estimator in $d=1000$ dimensions.}
    The distribution is a linear model with noise:
    covariates $x \sim \cN(0, I_d)$ and
response $y = \langle x, \beta \rangle + \cN(0, \sigma^2)$,
for $d = 1000, \sigma=0.1$, and $||\beta||_2 = 1$.
The estimator is $\hbeta = X^\dagger y$.
{\it Left:}
    Solid line shows mean over 50 trials,
    and individual points show a single trial.
{\it Right:}
Theoretical predictions for the bias, variance,
and risk from Claims~\ref{claim:main}
and \ref{claim:underparam}. }
    \label{fig:main}
\end{figure}

We focus on well-specified linear regression with Gaussian covariates, and we analyze the test risk of
the minimum-norm ridgeless regression estimator---
or equivalently, the estimator found by gradient descent on the least squares objective.
We show that as we increase the number of samples,
performance is non-monotonic: The test risk first decreases, and then \emph{increases}, before decreasing again. 

Such a ``double-descent'' behavior has been observed in the behavior of test risk as a function of the model size in a variety of machine learning settings
\cite{
opper1995statistical,
opper2001learning,
advani2017high,
belkin2018reconciling,
spigler2018jamming,
geiger2019jamming,
deep}.
Many of these works are motivated by understanding
the test risk as function of model size,
for a fixed number of samples.
In this work, we take a complementary view and understand the  
test risk as a function of
\emph{sample size}, for a fixed model.
We hope that understanding such simple settings
can eventually lead to understanding
the general phenomenon, 
and lead us to design learning algorithms which make the best use of data (and in particular, are monotonic in samples).

We note that similar analyses appear in recent works,
which we discuss in Section~\ref{sec:related}-- our focus is to highlight
the sample non-monotonicity implicit in these works, and give intuitions for the mechanisms behind it.
We specifically refer the reader to \cite{hastie2019surprises, mei2019generalization} for analysis in a setting most similar to ours.

\paragraph{Organization.}
We first define the linear regression setting in
Section~\ref{sec:linear}.
Then in Section~\ref{sec:analysis} we
state the form of the estimator found by gradient descent,
and give intuitions for why this estimator
has a peak in test risk when the number of samples is equal to the ambient dimension.
In Section~\ref{sec:biasvar}, we decompose the expected excess risk
into bias and variance contributions, and we state approximate expressions
for the bias, variance, and excess risk as a function of samples.
We show that these approximate theoretical predictions closely agree with
practice, as in Figure~\ref{fig:main}.

The peak in test risk turns out to be related to the conditioning of the data
matrix, and in Section~\ref{sec:conditioning} we give intuitions for
why this matrix is poorly conditioned in the ``critical regime'', but well
conditioned outside of it.
We also analyze the marginal effect of adding a single sample to
the test risk, in Section~\ref{sec:single}.
We conclude with discussion and open questions in Section~\ref{sec:discuss}.

\subsection{Related Works}
\label{sec:related}

This work was inspired by the long line of work studying ``double descent'' phenomena
in deep and shallow models.
The general principle is that as the model complexity increases,
the test risk of trained models first decreases and then increases (the standard U-shape), and then \emph{decreases again}.
The peak in test risk occurs in the ``critical regime'' when the models
are just barely able to fit the training set. The second descent occurs in the ``overparameterized regime'', when the model capacity is large enough to contain several interpolants on the training data. 
This phenomenon appears to be fairly universal among natural learning algorithms,
and is observed in simple settings such as linear regression, random features regression, classification with random forests, as well as modern neural networks.
Double descent of test risk with model size was introduced in generality
by \cite{belkin2018reconciling},
building on similar behavior observed as early as \cite{opper1995statistical, opper2001learning}
and more recently by
\cite{advani2017high, neal2018modern,spigler2018jamming, geiger2019jamming}.
A generalized double descent phenomenon 
was demonstrated on modern deep networks
by \cite{deep},
which also highlighted ``sample-wise nonmonotonicity''
as a consequence of double descent -- showing that more data can hurt for deep neural networks.

A number of recent works
theoretically analyze the double descent behavior
in simplified settings, often for linear models
\cite{belkin2019two, hastie2019surprises, bartlett2019benign, muthukumar2019harmless, bibas2019new, Mitra2019UnderstandingOP, mei2019generalization, liang2018just, liang2019risk,xu2019number,dereziski2019exact,lampinen2018analytic, deng2019model}.
At a high level, these works analyze the
test risk of estimators in overparameterized linear regression with different assumptions on the covariates.
We specifically refer the reader to
\cite{hastie2019surprises, mei2019generalization}
for rigorous analysis in a setting most similar to ours.
In particular, \cite{hastie2019surprises}
considers the asymptotic risk of the minimum norm
ridgeless regression estimator
in the limit where dimension $d$ and number of samples
$n$ are scaled as $d \to \infty, n = \gamma d$.
We instead focus on the sample-wise perspective: a fixed large $d$, but varying $n$.
In terms of technical content,
the analysis technique is not novel to our work,
and similar calculations appear in some of the prior works above.
Our main contribution is highlighting the sample non-monotonic behavior in a simple setting,
and elaborating on the mechanisms responsible.

While many of the above theoretical results are qualitatively similar, we highlight one interesting distinction:
our setting is \emph{well-specified}, 
and the bias of the estimator is monotone nonincreasing 
in number of samples (see Equation~\ref{eqn:bias}, and also \cite[Section 3]{hastie2019surprises}).
In contrast, for \emph{misspecified} problems (e.g. when the ground-truth is nonlinear, but we learn a linear model), the bias can actually increase with number of samples in addition to the variance increasing (see \cite{mei2019generalization}).


\section{Problem Setup}
\label{sec:linear}
Consider the following learning problem:
The ground-truth distribution
$\cD$ is $(x, y) \in \R^d \x \R$
with covariates $x \sim \cN(0, I_d)$ and
response $y = \langle x, \beta \rangle + \cN(0, \sigma^2)$
for some unknown, arbitrary
$\beta$ such that $||\beta||_2 \leq 1$.
That is, the ground-truth is an isotropic Gaussian with observation noise.
We are given $n$ samples $(x_i, y_i)$ from the distribution,
and we want to learn a linear model
$f_{\hbeta}(x) = \langle x, \hbeta \rangle$
for estimating $y$ given $x$.
That is, we want to find $\hbeta$
with small test mean squared error
\begin{align*}
R(\hbeta) &:=
\E_{(x, y) \sim \cD}[(\langle x, \hbeta \rangle - y)^2]\\
&= ||\hbeta - \beta||^2 + \sigma^2 \tag{for isotropic $x \sim \cN(0, I_d)$}
\end{align*}

Suppose we do this by performing ridgeless linear regression.
Specifically, we run gradient descent initialized at $0$ on the following objective (the empirical risk).

\begin{align}
    \min_{\hbeta} ||X\hbeta - y||^2
    \label{eqn:obj}
\end{align}
where $X \in \R^{n \x d}$ is the data-matrix of samples $x_i$,
and $y \in \R^n$ are the observations.

The solution found by gradient descent at convergence is
$\hbeta = X^\dagger y$, where $\dagger$ denotes the Moore–Penrose pseudoinverse\footnote{To see this,
notice that the iterates of gradient descent lie in the row-space of $X$.}.
Figure~\ref{fig:exp} plots the expected test MSE
of this estimator
$\E_{X, y}[R(\hbeta))]$ as we vary the number of train samples $n$.
Note that it is non-monotonic, with a peak in test MSE at $n=d$.

There are two surprising aspects of the test risk
in Figure~\ref{fig:exp}, in the overparameterized regime ($n < d$):
\begin{enumerate}
    \item {\bf The first descent:} where test risk initially decreases
even when we have less samples $n$ than dimensions $d$.
This occurs because the bias decreases.
\item {\bf The first ascent:}
where test risk increases, and peaks when $n=d$.
This is because the variance increases, and diverges when $n=d$.
\end{enumerate}
When $n > d$, this is the classical underparameterized regime, and test risk is monotone decreasing with number of samples.

Thus overparameterized linear regression exhibits a \emph{bias-variance tradeoff}:
bias decreases with more samples,
but variance can increase.
Below, we elaborate on the mechanisms and provide intuition for this non-monotonic behavior.

\section{Analysis}
\label{sec:analysis}
The solution found by gradient descent,
$\hbeta = X^\dagger y$,
has different forms depending on the ratio $n/d$.
When $n \geq d$, we are in the ``underparameterized'' regime
and there is a unique minimizer of the objective in Equation~\ref{eqn:obj}.
When $n < d$, we are ``overparameterized'' and there are many minimizers
of Equation~\ref{eqn:obj}.
In fact, since $X$ is full rank with probability 1,
there are many minimizers which \emph{interpolate}, i.e.
$X\hbeta = y$.
In this regime, gradient descent finds the minimum with
smallest $\ell_2$ norm $||\hat\beta||^2$.
That is, the solution can be written as

\begin{align*}
\hat\beta = X^\dagger y
=
\begin{dcases}
\argmin_{\substack{\beta: X\beta = y}} ||\beta||^2
    & \text{when $n \leq d$ \quad {\bf (``Overparameterized'')}}\\
\argmin_{\beta} ||X\beta - y||^2
    & \text{when $n > d$ \quad {\bf (``Underparameterized'')}}
\end{dcases}
\end{align*}

The overparameterized form yields insight into why the test MSE peaks at $n =
d$. Recall that the observations are noisy, i.e. $y = X\beta + \eta$ where $\eta
\sim N(0, \sigma^2 I_n)$.
When $n \ll d$, there are many interpolating estimators
$\{ \hbeta: X\hbeta = y \}$, and in particular there exist such $\hbeta$ with small
norm.
In contrast, when $n = d$, there is exactly one interpolating estimator
$(X\hbeta = y)$, but this estimator must have high norm in order to fit the
noise $\eta$.
More precisely, consider
\begin{align*}
    \hbeta &= X^\dagger y
    = X^\dagger(X\beta + \eta)
    = \underbrace{X^\dagger X\beta}_{\text{signal}} +
    \underbrace{X^\dagger \eta}_{\text{noise}}
\end{align*}
The signal term $X^\dagger X \beta$ is simply the orthogonal
projection of $\beta$ onto the rows of $X$.
When we are ``critically parameterized'' and $n \approx d$,
the data matrix $X$ is very poorly conditioned, and hence the noise term
$X^\dagger \eta$ has high norm, overwhelming the signal.
This argument is made precise in Section~\ref{sec:biasvar},
and in Section~\ref{sec:conditioning} we give intuition for why $X$ becomes poorly conditioned
when $n \approx d$.

The main point is that when $n = d$,
forcing the estimator $\hbeta$ to interpolate the noise
will force it to have very high norm, far from the ground-truth
$\beta$.
(See also Corollary 1 of \cite{hastie2019surprises}
for a quantification of this point).

\subsection{Excess Risk and Bias-Variance Tradeoffs}
\label{sec:biasvar}
For ground-truth parameter $\beta$,
the excess risk\footnote{For clarity, we consider the excess risk,
which omits the unavoidable additive $\sigma^2$ error in the true risk.}
of an estimator $\hbeta$ is:
\begin{align*}
\bar R(\hbeta) &:=
\E_{(x, y) \sim \cD}[(\langle x, \hbeta \rangle - y)^2]
- \E_{(x, y) \sim \cD}[(\langle x, \beta \rangle - y)^2]\\
&= \E_{x \sim \cN(0, I), \eta \sim \cN(0, \sigma^2)}[(\langle x, \hbeta \rangle - \langle x, \beta \rangle + \eta)^2]
- \sigma^2\\
&= ||\hbeta - \beta||^2  \
\end{align*}

For an estimator $\hbeta_{X, y}$ that is derived from samples $(X, y) \sim \cD^n$,
we consider the expected excess risk of $\hbeta = \hbeta_{X, y}$ in expectation over samples $(X, y)$ :
\begin{align}
\E_{X, y}[ \bar R(\hat\beta_{X, y}) ]
= \E_{X, y}[ ||\hat\beta - \beta||^2]
= \underbrace{ ||\beta - \E[\hat \beta]||^2}_{\text{Bias } B_n}
+ \underbrace{\E[|| \hbeta - \E[\hbeta] ||^2 ]}_{\text{Variance } V_n}
\end{align}
Where $B_n, V_n$ are the bias and variance of the estimator on $n$ samples.

For the specific estimator $\hbeta = X^\dagger y$ in the regime $n \leq d$,
the bias and variance can be written as (see Appendix~\ref{sec:biascomputation}):
\begin{align}
B_n &= ||\E_{X \sim \cD^n}[\Proj_{X^{\perp}}(\beta)]||^2 \label{eqn:bias}\\
V_n &=
\underbrace{\E_{X} [|| \Proj_X(\beta) - \E_X[\Proj_X(\beta)] ||^2]}_{(A)}
+
\sigma^2 \underbrace{\E_X[\Tr((XX^T)^{-1})]}_{(B)}
\label{eqn:variance}
\end{align}
where $\Proj_{X}$ is the orthogonal projector
onto the rowspace of the data $X \in \R^{n \x d}$,
and $\Proj_{X^\perp}$ is the projector onto the orthogonal complement of the rowspace.

From Equation~\ref{eqn:bias}, the bias is non-increasing with samples
($B_{n+1} \leq B_n$), since an additional sample can only grow
the rowspace: $X_{n+1}^{\perp} \subseteq X_n^{\perp}$.
The variance in Equation~\ref{eqn:variance} has two terms: the first term (A) is due to the randomness of $X$, and is bounded.
But the second term (B) is due to the randomness
in the noise of $y$, and diverges when $n \approx d$ since $X$ becomes poorly conditioned.
This trace term is responsible for the peak in test MSE at $n=d$.

We can also approximately compute the bias, variance, and excess risk.

\begin{claim}[Overparameterized Risk]
\label{claim:main}
Let $\gamma := \frac{n}{d} < 1$ be the underparameterization ratio.
The bias and variance are:
\begin{align}
B_n &= (1-\gamma)^2||\beta||^2\\
V_n &\approx \gamma(1-\gamma)||\beta||^2
+ \sigma^2 \frac{\gamma}{1-\gamma}
\end{align}
And thus the expected excess risk for $\gamma < 1$ is:
\begin{align}
\E[\bar R (\hbeta) ]
&\approx (1-\gamma)||\beta||^2
+ \sigma^2 \frac{\gamma}{1-\gamma}\\
&= (1-\frac{n}{d})||\beta||^2
+ \sigma^2 \frac{n}{d-n}
\end{align}
\end{claim}
These approximations are not exact because they hold
asyptotically in the limit of large $d$
(when scaling $n = \gamma d$),
but may deviate for finite samples.
In particular, the bias $B_n$ and term (A) of the variance can be computed exactly for finite samples: $\Proj_X$ is simply a projector onto a
uniformly random $n$-dimensional subspace, so $\E[ \Proj_X(\beta)] = \gamma \beta$, and similarly $\E[||\Proj_X(\beta)||^2] = \gamma ||\beta||^2$.
The trace term (B) is nontrivial to understand for finite samples, but converges\footnote{
For large $d$,
the spectrum of $(XX^T)$ is understood by the Marchenko–Pastur law \cite{marvcenko1967distribution}.
Lemma 3 of \cite{hastie2019surprises} uses this to show
that $\Tr((XX^T)^{-1}) \to \frac{\gamma}{1-\gamma}$.
}
to $\frac{\gamma}{1-\gamma}$
in the limit of large $n, d$
(e.g. Lemma 3 of~\cite{hastie2019surprises}).
In Section~\ref{sec:single}, we give intuitions for why the trace term
converges to this.

For completeness, the bias, variance, and excess risk in the underparameterized regime
are given in \cite[Theorem 1]{hastie2019surprises} as:
\begin{claim}[Underparameterized Risk, \cite{hastie2019surprises}]
\label{claim:underparam}
Let $\gamma := \frac{n}{d} > 1$ be the underparameterization ratio.
The bias and variance are:
\begin{align*}
B_n = 0 \quad, \quad
V_n \approx \frac{\sigma^2}{\gamma - 1}
\end{align*}
\end{claim}

Figure~\ref{fig:main} shows that Claims~\ref{claim:main} and~\ref{claim:underparam} agree
with the excess risk experimentally
even for finite $d=1000$.

\subsection{Conditioning of the Data Matrix}
\label{sec:conditioning}

Here we give intuitions for why the data matrix $X \in \R^{n \x d}$
is well conditioned for $n \ll d$, but has small singular values for $n \approx d$.

\subsubsection{Near Criticality}
First, let us consider the effect of adding a single sample when $n=(d-1)$.
For simplicity, assume the first $(d-1)$ samples $x_i$ are just the
standard basis vectors, scaled appropriately.
That is, assume the data matrix $X \in \R^{(d-1) \x d}$ is
$$
X = \mqty[ d I_{d-1} & 0].
$$
This has all non-zero singular values equal to $d$.
Then, consider adding a new isotropic Gaussian sample $x_{n+1} \sim \cN(0, I_d)$.
Split this into coordinates as $x_{n+1} = (g_1, g_2) \in \R^{d-1} \x \R$.
The new data matrix is
$$
X_{n+1} = \mqty[ d I_{d-1} & 0 \\
            g_1 & g_2]
$$
We claim that $X_{n+1}$ has small singular values.
Indeed, consider left-multiplication by $v^T := \mqty[g_1 & -d]$:
\begin{align*}
v^T X_{n+1} =
\mqty[g_1 & -d]
\mqty[ d I_{d-1} & 0 \\
        g_1 & g_2]
= \mqty[0 & -d g_2]
\end{align*}
Thus, $||v^T X_{n+1}||^2  \approx d^2$,
while $||v||^2 \approx 2d^2$. Since $X_{n+1}$ is full-rank,
it must have a singular value less than roughly $\frac{1}{\sqrt{2}}$.
That is, adding a new sample has shrunk the minimum non-zero singular value of $X$
from $d$ to less than a constant.

The intuition here is: although the new sample $x_{n+1}$ adds rank to the existing samples, it does so in a very fragile way.
Most of the $\ell_2$ mass of $x_{n+1}$ is contained in the span of existing samples, and $x_n$ only contains a small component outside of this subspace. This causes $X_{n+1}$ to have small singular values,
which in turn causes the ridgeless regression estimator (which applies $X^\dagger$) to be sensitive to noise.

A more careful analysis shows that the singular values are actually even smaller than the above simplification suggests
--- since in the real setting, the matrix $X$ was already poorly conditioned even before the new sample $x_{n+1}$.
In Section~\ref{sec:single}
we calculate the exact effect of adding a single sample to the excess risk.

\subsubsection{Far from Criticality}
When $n \ll d$, the data matrix $X$ does not have singular values close to $0$.
One way to see this is to notice that
since our data model treats features and samples symmetrically,
$X$ is well conditioned in the regime $n \ll d$
for the same reason that standard linear regression works in the classical underparameterized regime $n \gg d$ (by ``transposing'' the setting).

More precisely, since $X$ is full rank, its smallest non-zero singular value
can be written as
$$\sigma_{\text{min}}(X) =
\min_{v \in \R^n: ||v||_2 = 1} ||v^T X||_2$$
Since $X$ has entries i.i.d $\cN(0, 1)$,
for every fixed vector $v$ we have $\E_X[ ||v^T X||^2 ] = d ||v||^2 = d$.
Moreover, for $d = \Omega(n)$ uniform convergence holds, and
$||v^T X||^2$ concentrates around its expectation for
all vectors $v$ in the $\ell_2$ ball. Thus:
$$\sigma_{\text{min}}(X)^2
\approx
\E_X\left[ \min_{v \in \R^n: ||v||_2 = 1} ||v^T X||^2 \right]
\approx
\min_v \E_X[ ||v^TX||^2 ] 
= d
$$


\subsection{Effect of Adding a Single Sample}
\label{sec:single}
Here we show how the
trace term of the variance in Equation~\ref{eqn:variance} changes with increasing samples.
Specifically, the following claim shows how $\Tr((XX^T)^{-1})$ grows
when we add a new sample to $X$.

\begin{claim}
\label{claim:trace}
Let $X \in \R^{n \x d}$ be the data matrix after $n$ samples,
and let $x \in \R^d$ be the $(n+1)$th sample.
The new data matrix is $X_{n+1} = \mqty[X \\ x^T]$, and
$$
\Tr((X_{n+1}X_{n+1}^T)^{-1})
=
\Tr[(XX^T)^{-1}]
+ \frac{1 + ||(X^T)^\dagger x||^2}{|| \mathrm{Proj}_{X^\perp}(x)||^2}
$$
\end{claim}
\begin{proof}
By computation in Appendix~\ref{app:trace}.
\end{proof}

If we heuristically assume the denominator
concentrates around its expectation,
$||\Proj_{X^\perp}(x)||^2 \approx d-n$, then we can
use Claim~\ref{claim:trace}
to estimate the expected effect of a single sample:

\begin{align}
\E_{x} \Tr((X_{n+1}X_{n+1}^T)^{-1})
&\approx
\Tr[(XX^T)^{-1}]
+ \frac{1 + \E_{x}||(XX^T)^{-1}Xx||^2}{d-n}\\
&=
\Tr[(XX^T)^{-1}]\left(1 + \frac{1}{d-n}\right)
+ \frac{1}{d-n}
\label{eqn:tr_single}
\end{align}

We can further estimate the growth by taking a continuous
limit for large $d$.
Let $F(\frac{n}{d}) := \E[ \Tr((X_{n}X_{n}^T)^{-1})]$.
Then for $\gamma := \frac{n}{d}$, Equation~\ref{eqn:tr_single} yields the differential equation
\begin{align*}
\frac{d F(\gamma)}{d\gamma}
= (1-\gamma)^{-1} F + (1-\gamma)^{-1}
\end{align*}
which is solved by $F(\gamma) = \frac{\gamma}{1-\gamma}$.
This heuristic derivation that
$\E[\Tr(XX^T)^{-1}] \to  \frac{\gamma}{1-\gamma}$
is consistent with the rigorous asymptotics given in \cite[Lemma 3]{hastie2019surprises}
and used in Claim \ref{claim:main}.


\section{Discussion}
\label{sec:discuss}
We hope that understanding such simple settings
can eventually lead to understanding the general
behavior of overparameterized models
in machine learning.
We consider it extremely unsatisfying that
the most popular technique in modern machine learning
(training an overparameterized neural network with SGD)
can be nonmonotonic in samples~\cite{deep}.
We hope that a greater understanding here
could help develop learning algorithms which make
the best use of data (and in particular, are monotonic in samples).

In general, we believe it is interesting to understand when and why learning
algorithms are monotonic -- especially when we don't explicitly enforce them to
be.

\subsection*{Acknowledgements}
We especially thank
Jacob Steinhardt and 
Aditi Raghunathan
for discussions and suggestions that motivated this work.
We thank
Jaros\l aw B\l asiok,
Jonathan Shi,
and Boaz Barak for useful discussions
throughout this work,
and we thank Gal Kaplun and Benjamin L. Edelman for feedback on an early
draft.

This work supported in part by
supported by NSF awards CCF 1565264,
CNS 1618026,
and 
CCF 1715187,
a Simons Investigator Fellowship,
and a Simons Investigator Award.

\bibliographystyle{apalike}
\bibliography{refs}

\appendix

\section{Appendix: Computations}
\subsection{Bias and Variance}
The computations in this section are standard.

\label{sec:biascomputation}
Assume the data distribution and problem setting from Section~\ref{sec:linear}.

For samples $(X, y)$, the estimator is:
\begin{align}
    \hbeta = X^\dagger y
    =
\begin{cases}
X^T(XX^T)^{-1}y & \text{when $n \leq d$}\\
(X^TX)^{-1}X^Ty & \text{when $n > d$}
\end{cases}
\end{align}

\begin{lemma}
For $n \leq d$, the bias and variance of the estimator $\hbeta = X^\dagger y$ is
\begin{align*}
B_n &= ||\E_{X \sim \cD^n}[\Proj_{X^{\perp}}(\beta)]||^2 \\
V_n &=
\underbrace{\E_{X} [|| \Proj_X(\beta) - \E_X[\Proj_X(\beta)] ||^2]}_{(A)}
+
\sigma^2 \underbrace{\E_X[\Tr((XX^T)^{-1})]}_{(B)}
\end{align*}
\end{lemma}
\begin{proof}

{\bf Bias.}
Note that
\begin{align*}
\beta - \E[\hbeta] &=
\beta - \E_{X, \eta}[ X^T(XX^T)^{-1}(X\beta + \eta) ]\\
&= \E_X[(I -  X^T(XX^T)^{-1}X)\beta]\\
&= \E_X[Proj_{X^\perp}(\beta)]\\
\end{align*}

Thus the bias is
\begin{align*}
B_n &= ||\beta - \E[\hat \beta]||^2\\
&= ||E_{X_n}[Proj_{X_n^{\perp}}(\beta)]||^2
\end{align*}

{\bf Variance.}

\begin{align*}
V_n &= \E_{\hbeta}[|| \hbeta - \E[\hbeta] ||^2 ]\\
&= \E_{X,\eta}[
||
X^T(XX^T)^{-1}(X\beta + \eta)
-\E_{X}[X^T(XX^T)^{-1}X\beta]
||^2 ]\\
&=
\E_{X, \eta} [||
(S - \bar{S})\beta
+
X^T(XX^T)^{-1}\eta
||^2 ] \tag{$S := X^T(XX^T)^{-1}X, \bar{S} := \E[S]$}\\
&=
\E_{X} [||
(S - \bar{S})\beta||^2]
+
\E_{X, \eta}[||X^T(XX^T)^{-1}\eta ||^2 ]\\
&=
\E_{X} [||
(S - \bar{S})\beta||^2]
+
\sigma^2 Tr((XX^T)^{-1})
\end{align*}

Notice that $S$ is projection onto the rowspace of $X$, i.e. $S = Proj_{X}$.
Thus,

\begin{align*}
V_n
:=
\E_{X} [|| Proj_X(\beta) - \E_X[Proj_X(\beta)] ||^2]
+
\sigma^2 Tr((XX^T)^{-1})
\end{align*}

\end{proof}

\subsection{Trace Computations}
\label{app:trace}
\begin{proof}[Proof of Claim~\ref{claim:trace}]
Let $X \in \R^{n \x d}$ be the data matrix after $n$ samples,
and let $x \in \R^d$ be the $(n+1)$th sample.
The new data matrix is $X_{n+1} = \mqty[X \\ x^T]$, and
$$
X_{n+1}X_{n+1}^T
=
\mqty[XX^T & Xx\\
x^TX^T & x^Tx]
$$

Now by Schur complements:

\begin{align*}
(X_{n+1}X_{n+1}^T)^{-1}
&=
\mqty[XX^T & Xx\\
x^TX^T & x^Tx]^{-1}\\
&=
\mqty[
(XX^T - \frac{Xxx^TX^T}{||x||^2})^{-1}
& -(XX^T - \frac{Xxx^TX^T}{||x||^2})^{-1}\frac{Xx}{||x||^2}\\
\dots
&
(||x||^2 - x^TX^T(XX^T)^{-1}Xx )^{-1}
]
\end{align*}

Thus
\begin{align*}
Tr((X_{n+1}X_{n+1}^T)^{-1})
&=
Tr((XX^T - \frac{Xxx^TX^T}{||x||^2})^{-1})
+
(||x||^2 - x^TX^T(XX^T)^{-1}Xx )^{-1}\\
&=
Tr((XX^T - \frac{Xxx^TX^T}{||x||^2})^{-1})
+
(x^T(x - Proj_X(x)) )^{-1}\\
&=
Tr((XX^T - \frac{Xxx^TX^T}{||x||^2})^{-1})
+
\frac{1}{||Proj_{X^\perp}(x)||^2}
\end{align*}

By Sherman-Morrison:
\begin{align*}
(XX^T - \frac{Xxx^TX^T}{||x||^2})^{-1}
&=
(XX^T)^{-1} + \frac{(XX^T)^{-1}Xxx^TX^T(XX^T)^{-1}}{||x||^2 - x^TX^T(XX^T)^{-1}Xx}\\
\implies
Tr (XX^T - \frac{Xxx^TX^T}{||x||^2})^{-1}
&=
Tr[(XX^T)^{-1}]
+ \frac{||(XX^T)^{-1}Xx||^2}{||Proj_{X^\perp}(x)||^2}
\end{align*}

Finally, we have

\begin{align*}
Tr((X_{n+1}X_{n+1}^T)^{-1})
&=
Tr[(XX^T)^{-1}]
+ \frac{1 + ||(XX^T)^{-1}Xx||^2}{||Proj_{X^\perp}(x)||^2}
\end{align*}
or equivalently:

$$
Tr((X_{n+1}X_{n+1}^T)^{-1})
=
Tr[(XX^T)^{-1}]
+ \frac{1 + ||(X^T)^\dagger x||^2}{||Proj_{X^\perp}(x)||^2}
$$

or equivalently:
$$
Tr((X_{n+1}X_{n+1}^T)^{-1})
=
Tr[(XX^T)^{-1}]
+ \frac{1 + ||\gamma||^2}{||Proj_{X^\perp}(x)||^2}
\quad \text{ where }
\gamma := \argmin_v || X^T v - x ||^2
$$
\end{proof}
\end{document}